\documentclass[a4paper]{article}
\usepackage[latin1]{inputenc} % ou \usepackage[utf8]{inputenc}
\usepackage[T1]{fontenc} % ou \usepackage[OT1]{fontenc}
\usepackage{RR,RRthemes}
\usepackage{hyperref}

\usepackage{amsmath, amsthm, amssymb}
\usepackage{amsfonts}
\usepackage{enumerate}
\usepackage{graphics} %cx na konci pro PDF
\usepackage{graphicx}
\usepackage{algorithm}
\usepackage{algorithmic}
\usepackage{verbatim}

\RRdate{Février 2011}
\RRauthor{% les auteurs
   Milan Hlad\'{i}k\thanks{
    Charles University, Faculty  of  Mathematics  and  Physics, Department of Applied Mathematics, Malostransk\'e n\'am.~25, 11800, Prague, Czech Republic, e-mail: \texttt{milan.hladik@matfyz.cz}}
  \thanks[inr]{
    INRIA Sophia-Antipolis M\'editerran\'ee, 2004 route des Lucioles, BP 93, 06902
    Sophia-Antipolis Cedex, France, e-mail: \texttt{FirstName.LastName(AT)inria.fr}}
  \and 
  David Daney\thanksref{inr}
  \and 
  Elias Tsigaridas\thanks{Computer Science Department, Aarhus University, Denmark.
  e-mail: \texttt{elias@cs.au.dk}}
}
\authorhead{Hlad\'{i}k \& Daney \& Tsigaridas}
\RRtitle{Caractérisation et approximation des ensembles de valeurs propres des matrices symétriques intervalles}
\RRetitle{Characterizing and approximating eigenvalue sets of symmetric interval matrices}
\titlehead{}
\RRresume{We consider the eigenvalue problem for the case where the input
  matrix is symmetric and its entries perturb in some given
  intervals. We present a characterization of some of the exact
  boundary points, which allows us to introduce an inner approximation
  algorithm, that in many case estimates exact bounds.  To our
  knowledge, this is the first algorithm that is able to guarantee
  exactness.  We illustrate our approach by
  several examples and numerical experiments.
}
\RRabstract{  We consider the eigenvalue problem for the case where the input
  matrix is symmetric and its entries perturb in some given
  intervals. We present a characterization of some of the exact
  boundary points, which allows us to introduce an inner approximation
  algorithm, that in many case estimates exact bounds.  To our
  knowledge, this is the first algorithm that is able to guarantee
  exactness.  We illustrate our approach by
  several examples and numerical experiments.}
\RRmotcle{Matrice d'intervalles, matrice symétrique, analyse par intervalles, valeur propre, borne sur les valeurs propres}
\RRkeyword{Interval matrix, symmetric matrix, interval analysis, eigenvalue, eigenvalue bounds}
\RRprojet{Coprin}  % cas d'un seul projet
\RRdomaine{4} % cas du domaine numero 1
\RRthemeProj{coprin} % theme du projet Apics
\RRdomaineProjBis{coprin} % domaine du projet Apics
\RCSophia % Sophia Antipolis M\'editerran\'ee

\newcommand{\tluste}[1]{\mbox{\mathversion{bold}$ #1 $}}
\newcommand{\vr}[1]{{{#1}}}

\newcommand{\mace}[1]{{{#1}}}
\newcommand{\mna}[1]{{\mathcal{#1}}}

\newcommand{\omace}[1]{\mbox{$\overline{\mace{#1}}$}} 
\newcommand{\umace}[1]{\mbox{$\underline{\mace{#1}}$}} 
\newcommand{\imace}[1]{\mbox{$\tluste{#1}$}} 
\newcommand{\simace}[1]{\mbox{$\tluste{\scriptstyle #1}$}} 
\newcommand{\smace}[1]{\tluste{#1}{}^S} 
\def\Mid#1{{#1}_c}
\def\Rad#1{{#1}_\Delta}
\newcommand{\ovr}[1]{\mbox{$\overline{\vr{#1}}$}} 
\newcommand{\uvr}[1]{\mbox{$\underline{\vr{#1}}$}}
\newcommand{\ov}[1]{\mbox{$\overline{{#1}}$}} 
\newcommand{\uv}[1]{\mbox{$\underline{{#1}}$}}

\newcommand{\ivr}[1]{\mbox{$\tluste{#1}$}} 

\newcommand{\inum}[1]{\mbox{$\tluste{#1}$}} 
\def\keyword#1{{\mbox{$\mathbf{#1}$}}}

\newcommand{\R}[0]{{\mathbb{R}}}
\def\L{{\Lambda}}
\def\inner{\mu}
\def\outer{\omega}

\def\conv{\mathrm{conv}\,}

\newcommand{\seznam}[1]{{\{1, \ldots, {#1}\}}}
\newcommand{\sseznam}[2]{{\{{#1}, \ldots, {#2}\}}}
\def\clqq{``}
\def\crqq{''}
\def\quo#1{\clqq{}#1\crqq{}}  % snadny zapis ang. uvozovek

\newcommand{\maxim}[2]{{\max {\{#1; \ {} #2 \}} }}

\def\sgn#1{\mathop{\mathrm{sgn}(#1)}}
\def\diag#1{\mathop{\mathrm{diag}(#1)}}

\def\nref#1{$(\ref{#1})$}

\begin{document}
 \RRNo{7544}
\makeRR   % cas d'un rapport de recherche
%% \makeRT % cas d'un rapport technique.
%% a partir d'ici, chacun fait comme il le souhaite

%definition of Theorems, ...
\newtheorem{theorem}{Theorem}
\newtheorem{assertion}{Assertion}
\newtheorem{proposition}{Proposition}
\newtheorem{lemma}{Lemma}
\newtheorem{corollary}{Corollary}
\newtheorem{conjecture}{Conjecture}
\theoremstyle{definition}
\newtheorem{alg}{Algorithm}
\newtheorem{definition}{Definition}
\newtheorem{example}{Example}
\newtheorem{remark}{Remark}

%%%%%%%%%%%%%%%%%%%%%%%%%%%%%%%%%%%%%%%%%%%%%%%%%%%%%%%%%%%%%%%
% INTRODUCTION
%%%%%%%%%%%%%%%%%%%%%%%%%%%%%%%%%%%%%%%%%%%%%%%%%%%%%%%%%%%%%%%
\section{Introduction}

Computing eigenvalues of a matrix is a basic linear algebraic task
used throughout in mathematics, physics and computer science. Real
life makes this problem more complicated by imposing uncertainties and
measurement errors on the matrix entries. We suppose we are given some
compact intervals in which the matrix entries can perturb. The set of
all possible real eigenvalues forms a compact set, and the question
that we deal with in this paper is how to characterize and compute it.

The interval eigenvalue problem has its own history. The first results are
probably due to Deif \cite{Dei1991} and Rohn \& Deif
\cite{RohDei1992}: bounds for real and imaginary parts for complex
eigenvalues were studied by Deif \cite{Dei1991}, while Rohn \& Deif
\cite{RohDei1992} considered real eigenvalues. Their theorems are
applicable only under an assumption on sign pattern invariancy of
eigenvectors, which is not easy to verify (cf. \cite{DeiRoh1994}). A
boundary point characterization of eigenvalue set was given by Rohn
\cite{Roh1993}, and it was used by Hlad\'{i}k et al. \cite{HlaDan2008}
to develop a branch \& prune algorithm producing an arbitrarily tight
approximation of the eigenvalue set. Another approximate method was
given by Qiu et al. \cite{QiuMulFro2001}. The related topic of finding
verified intervals of eigenvalues for real matrices was studied
in e.g.~\cite{AleMay2000}.

In this paper we focus on the case of the symmetric eigenvalue
problem. Symmetric matrices naturally appear in many practical
problems, but symmetric interval matrices are hard to deal with.  This
is so, mainly due to the so-called dependencies, that is, correlations
between the matrix components. If we \quo{forget} these dependencies
and solve the problem by reducing it to the previous case, then the
results will be greatly overestimated, in general (but not the extremal points,
see Theorem~\ref{thmSymUnsym}). From now on we consider only the
symmetric case.

Due to the dependencies just mentioned, the theoretical background for the
eigenvalue problem of symmetric interval matrices is not wide enough and
there are few practical methods. The known results are that by Deif
\cite{Dei1991} and Hertz \cite{Her1992}. Deif \cite{Dei1991} gives an
exact description of the eigenvalue set together with restrictive its
assumptions. Hertz \cite{Her1992} (cf. \cite{Roh2005}) proposed a
formula for computing two extremal points of the eigenvalue set---the
largest and the smallest one. As the problem itself is very hard, it
is not surprising conjectures on the problem \cite{QiuWan2005}
turned out to be wrong \cite{YuaHe2008}.

In the recent years, several approximation algorithms have been
developed.  An evolution strategy method by Yuan et
al. \cite{YuaHe2008} yields an inner approximation of the eigenvalue
set. By means of matrix perturbation theory, Qiu et
al. \cite{QiuChe1996} proposed an algorithm for approximate bounds,
and Leng \& He \cite{LenHe2007} for an outer estimation. An outer estimation
was also considered by Kolev~\cite{Kol2006}, but for general case with
nonlinear dependencies. Some initial bounds that are easy and fast to
compute were discussed by Hlad\'{i}k et al. \cite{HlaDan2008b}. An
iterative algorithm for outer estimation was given by Beaumont
\cite{Bea2000}.

This problem has a lot of applications in the field of mechanics and
engineering. Let us mention for instance automobile suspension systems
\cite{QiuMulFro2001}, mass structures \cite{QiuChe1996}, vibrating
systems \cite{Dim1995}, principal component analysis
\cite{GioLau2006}, and robotics \cite{chablat04}. Another applications
arise from the engineering area concerning singular values and
condition numbers. Using the well-known Jordan--Wielandt
transformation \cite{GolLoa1996, HorJoh1985, Mey2000} we can simply
reduce a singular value calculation to a symmetric eigenvalue one.

The rest of the paper is structured as follows: In
Sec.~\ref{sec:prelim} we introduce the notation that we use throughout
the paper.  In Sec.~\ref{sec:main-th} we present our main theoretical
result, and in Sec.~\ref{sAlg} we develop our algorithms for the
problem. Finally, in Sec.~\ref{sNumer} we demonstrate
our approach by a number of examples and numerical experiments, 
and we conclude in Sec.~\ref{sec:conclusion}.

%%%%%%%%%%%%%%%%%%%%%%%%%%%%%%%%%%%%%%%%%%%%%%%%%%%%%%%%%%%%%%%
% THEORETICAL BACKGROUND
%%%%%%%%%%%%%%%%%%%%%%%%%%%%%%%%%%%%%%%%%%%%%%%%%%%%%%%%%%%%%%%
\section{Basic definitions and theoretical background}\label{sec:prelim}

Let us introduce some notions first. An interval matrix is defined as 
$$\imace{A}:=[\umace{A},\omace{A}]=\{\mace{A}\in \R^{m \times n};\; \umace{A}\leq\mace{A}\leq\omace{A}\},$$
where $\umace{A},\,\omace{A}\in\R^{m\times n}$,
$\umace{A}\leq\omace{A}$, are given matrices. By
\begin{align*}
  \mace{A}_c &:= \frac{1}{2}(\umace{A}+\omace{A}),\quad 
  \mace{A}_\Delta := \frac{1}{2}(\omace{A}-\umace{A})
\end{align*}
we denote the midpoint and the radius of $\imace{A}$, respectively.

By an interval linear system of equations $\imace{A}\vr{x}=\ivr{b}$ we
mean a family of systems $\mace{A}\vr{x}=\vr{b}$, such that 
$\mace{A}\in\imace{A}$, $\vr{b}\in\ivr{b}$. In a similar way we
introduce interval linear systems of inequalities and mixed systems of
equations and inequalities. A vector $\vr{x}$ is a solution of
$\imace{A}\vr{x}=\ivr{b}$ if it is a solution of
$\mace{A}\vr{x}=\vr{b}$ for some $\mace{A}\in\imace{A}$ and
$\vr{b}\in\ivr{b}$.
% Solution of another kinds of interval systems is defined analogously.
We assume that the reader is familiar with the basics of interval
arithmetic; for further details we refer to e.g.~\cite{AleHer1983,HanWal2004,Neu1990}.

Let $\mna{F}$ be a family of $n\times n$ matrices. 
We denote the eigenvalue set of the family $\mna{F}$ by 
$$\L(\mna{F}):=\{\lambda\in\R;\; \mace{A}\vr{x}=\lambda\vr{x},\ \vr{x}\not=\vr{0},\ \mace{A}\in\mna{F} \}.$$

\emph{A symmetric interval matrix} as defined as 
$$\smace{A}:=\{\mace{A}\in\imace{A}\mid\mace{A}=\mace{A}^T\}.$$
It is usually a proper subset of $\imace{A}$.
Considering the eigenvalue set $\L(\imace{A})$, it, generally, represents an
overestimation of $\L(\smace{A})$. That is why we focus directly on
eigenvalue set of the symmetric counterpart, even though we must take
into account the dependencies between the elements, in the definition of $\smace{A}$.

A real symmetric matrix $\mace{A}\in\R^{n\times n}$ has always $n$
real eigenvalues, let us sort them in a non-increasing order
$$
\lambda_1(\mace{A})\geq\lambda_2(\mace{A})
\geq\dots\geq\lambda_n(\mace{A}).
$$
We extend this notation for symmetric interval matrices
$$\ivr{\lambda}_i(\smace{A}):=\big\{\lambda_i(\mace{A})\mid\mace{A}\in\smace{A}\big\}.$$
These sets represent $n$ compact intervals $\inum{\lambda}_i(\smace{A})=[\uv{\lambda}_i(\smace{A}),\ov{\lambda}_i(\smace{A})]$, $i=1,\dots,n$; cf. \cite{HlaDan2008b}. 
%the fact follows from the 
%continuity of eigenvalues and compactness of $\smace{A}$.
%\footnote{@Milan: Should we refer to our SIMAX paper for this?}  
The intervals can be disjoint, can overlap, or some of them, can be
identical.  However, what it can not happen is one interval to be a
proper subset of another interval.  The union of these intervals
produces $\L(\smace{A})$.

Throughout the paper we use the following notation:
%
%\subsection*{Notation}
%
\begin{center}
\smallskip
\begin{tabular*}{0.85\textwidth}{lp{0.75\textwidth}}
 $\lambda_i(\mace{A})$ & 
   the $i$th eigenvalue of a symmetric matrix $\mace{A}$
   (in a non-increasing order)\\[1mm]
 $\sigma_i(\mace{A})$ & 
   the $i$th singular value of a matrix $\mace{A}$
   (in a non-increasing order)\\[1mm]
 $\vr{v}_i(\mace{A})$ & 
   the eigenvector associated to the $i$th eigenvalue of
   a symmetric matrix $\mace{A}$\\[1mm]
 $\rho(\mace{A})$ & 
   the spectral radius of a matrix $\mace{A}$ \\[1mm]
 $\partial\mna{S}$  & 
   the boundary of a set $\mna{S}$  \\[1mm]
 $\conv\mna{S}$  & 
   the convex hull of a set $\mna{S}$  \\[1mm]
 $\diag{\vr{y}}$ & 
   the diagonal matrix with entries $y_1,\dots,y_n$ \\[1mm]
 $\sgn{\vr{x}}$ & 
   the sign vector of a vector $x$, i.e., 
   $\sgn{\vr{x}}=(\sgn{x_1},\dots,\sgn{x_n})^T$ \\[1mm]
 $\|\vr{x}\|_2$ &
   the Euclidean vector norm, i.e., 
   $\|\vr{x}\|_2=\sqrt{\vr{x}^T\vr{x}}$ \\[1mm]
 $\|\vr{x}\|_\infty$ &
   the Chebyshev (maximum) vector norm, i.e., 
   $\|\vr{x}\|_\infty=\maxim{|x|_i}{i=1,\dots,n}$ \\[1mm]
 $\vr{x}\leq\vr{y}$, $\mace{A}\leq\mace{B}$ &
   vector and matrix relations are understood component-wise \\[1mm]
\end{tabular*}
\end{center}

\section{Main theorem}
\label{sec:main-th}

The following theorem is the main theoretical result of the the present paper.
We remind the reader that the principal $m\times m$ submatrix of a
given $n \times n$ matrix is any submatrix obtained by eliminating any
$n-m$ rows and the same $n-m$ columns.

\begin{theorem}\label{thmBdDes}
  Let $\lambda\in\partial \L(\smace{A})$. There is a principal
  submatrix $\smace{\widetilde{A}}\in\R^{k\times k}$ of $\smace{A}$
  such that:
  \begin{itemize}
  \item If $\lambda=\ovr{\lambda}_j(\smace{A})$ for some $j\in\seznam{n}$, then 
    \begin{align}\label{bdLamUpper}
      \lambda\in
      \big\{\lambda_i\big(\Mid{\widetilde{A}}+
      \diag{\vr{z}}\Rad{\widetilde{A}}\diag{\vr{z}}\big);\ 
      \vr{z}\in\{\pm1\}^k,\ i=1,\dots,k\big\}
      \enspace .
    \end{align}
  \item  If $\lambda=\uvr{\lambda}_j(\smace{A})$ for some $j\in\seznam{n}$, then
    \begin{align}\label{bdLamLower}
      \lambda\in
      \big\{\lambda_i\big(\Mid{\widetilde{A}}-
      \diag{\vr{z}}\Rad{\widetilde{A}}\diag{\vr{z}}\big);\ 
      \vr{z}\in\{\pm1\}^k,\ i=1,\dots,k\big\}
      \enspace .
    \end{align}
  \end{itemize}
%\footnote{@Milan: What is $\ovr{\lambda}_j$ and $\uvr{\lambda}_j$?}
\end{theorem}

\begin{proof}
  Let $\lambda\in\partial \L(\smace{A})$.
  Then either
  $\lambda=\ovr{\lambda}_j(\smace{A})$ or
  $\lambda=\uvr{\lambda}_j(\smace{A})$,
  for some $j\in\seznam{n}$. 
  We assume the former case. The latter could be proved similarly.
  % Without loss of generality weassume the first case,
  % that is, $\lambda=\ovr{\lambda}_j(\smace{A})$ for some
  % $j\in\seznam{n}$. The second case can be proven analogously.

  The eigenvalue $\lambda$ is achieved for a matrix
  $\mace{A}\in\imace{A}$. It is without loss of generality to assume
  that the corresponding eigenvector $\vr{x}$, $\|\vr{x}\|_2=1$, is of
  the form $\vr{x}=(\vr{0}^T,\vr{y}^T)^T$, where $\vr{y}\in\R^k$ and
  $y_i\not=0$, for all $1 \leq i \leq k$, and for some $k \in \{1, \dots,
  n\}$. The symmetric interval matrix $\smace{A}$ can be written as
  \begin{align*}
    \smace{A}=
    \begin{pmatrix}\smace{B}&\imace{C}\\
      \imace{C}^T&\smace{D}\end{pmatrix},
  \end{align*}
  where $\smace{B}\subset\R^{(n-k)\times(n-k)}$,
  $\imace{C}\subset\R^{(n-k)\times k}$, $\smace{D}\subset\R^{k\times
    k}$.  

  From the basic equality $\mace{A}\vr{x}=\lambda\vr{x}$ it
  follows that
  \begin{align} \label{pfLamC}
    \mace{C}\vr{y}=\vr{0} \  \mbox{ for some }\mace{C}\in\imace{C},
  \end{align}
  and
  \begin{align}\label{pfLamD}
    \mace{D}\vr{y}=\lambda\vr{y} \ 
    \mbox{ for some }\mace{D}\in\smace{D}.
  \end{align}
  We focus on the latter relation; it says that
  $\lambda$ is an eigenvalue of $\mace{D}$. We will show that
  $\smace{D}$ is the required principal submatrix
  $\smace{\widetilde{A}}$ and $\mace{D}$ could be written as in \nref{bdLamUpper}.

  From \nref{pfLamD} we have that $\lambda=\vr{y}^T\mace{D}\vr{y}$, 
  and hence the partial derivatives are
  \begin{align*}
    \frac{\partial\lambda}{\partial{d}_{ij}}
    =y_iy_j\not=0, \quad i,j=1,\dots,k.
  \end{align*}
  This relation strongly influences the structure of $\mace{D}$.  
  If $y_iy_j>0$, then $d_{ij}=\ov{d}_{ij}$. This is so, because otherwise by increasing
  $d_{ij}$ we also increase the value of $\lambda$, which contradicts
  our assumption that $\lambda$ lies on the upper boundary of
  $\L(\smace{A})$. Likewise, $y_iy_j<0$ implies
  $d_{ij}=\uv{d}_{ij}$. 
  This allows us to write $\mace{D}$ in the following more compact form
  \begin{align}\label{pfFormD}
    \mace{D}=\Mid{D}+\diag{\vr{z}}\Rad{D}\diag{\vr{z}},
  \end{align}
  where $\vr{z}=\sgn{\vr{y}}\in\{\pm1\}^k$.  Therefore, $\lambda$
  belongs to a set as the one presented in the right-hand side of
  \nref{bdLamUpper}, which completes the proof.
\end{proof}

Note that not every $\uvr{\lambda}_j(\smace{A})$ or
$\ovr{\lambda}_j(\smace{A})$ is a boundary point of
$\L(\smace{A})$. Theorem~\ref{thmBdDes} is also true for such 
$\uvr{\lambda}_j(\smace{A})$ or $\ovr{\lambda}_j(\smace{A})$ that are
non-boundary, but make no multiple eigenvalue (since the corresponding
eigenvector is uniquely determined). However, truthfulness of
Theorem~\ref{thmBdDes} for all $\uvr{\lambda}_j(\smace{A})$ and
$\ovr{\lambda}_j(\smace{A})$, $j=1,\dots,n$, is still an open
question. Moreover, full characterization of all
$\uvr{\lambda}_j(\smace{A})$ and $\ovr{\lambda}_j(\smace{A})$,
$j=1,\dots,n$, is lacking, too.

As we have already mentioned, in general, the eigenvalue set of an
interval matrix is larger than the eigenvalue set of its symmetric
counterpart. This is true even if both the midpoint and radius
matrices are symmetric (see Example~\ref{exSymUnsym}). The following
theorem says that overestimation caused by the additional matrices is
somehow limited by the intermediate area.

\begin{theorem}\label{thmSymUnsym} 
  Let $\Mid{A},\Rad{A}\in\R^{n\times n}$ be symmetric matrices. Then
  \begin{displaymath} 
    \conv \L(\smace{A})=\conv \L(\imace{A}).
  \end{displaymath}
\end{theorem}
\begin{proof} 
  The inclusion $\conv \L(\smace{A})\subseteq\conv \L(\imace{A})$
  follows from the definition of the convex hull.

  Let $\mace{A}\in\imace{A}$ be arbitrary, $\lambda$ one of its real eigenvalues,
  and $\vr{x}$ the corresponding eigenvector, where $\|\vr{x}\|_2=1$.
  Let $\mace{B}:=\frac{1}{2}(\mace{A}+\mace{A}^T)\in\smace{A}$,
  then the  following holds:
  \begin{align*} \lambda =\vr{x}^T\mace{A}\vr{x}
    \leq\max_{\|\vr{y}\|_2=1}\vr{y}^T\mace{A}\vr{y}
    =\max_{\|\vr{y}\|_2=1}\vr{y}^T\mace{B}\vr{y}
    =\lambda_1(\mace{B})\in\conv \L(\smace{A}).
  \end{align*} 
  Similarly,
  \begin{align*} \lambda =\vr{x}^T\mace{A}\vr{x}
    \geq\min_{\|\vr{y}\|_2=1}\vr{y}^T\mace{A}\vr{y}
    =\min_{\|\vr{y}\|_2=1}\vr{y}^T\mace{B}\vr{y}
    =\lambda_n(\mace{B})\in\conv \L(\smace{A}).
  \end{align*} 
  Therefore $\lambda\in\conv\L(\smace{A})$,
  and so 
  $\conv \L(\imace{A}) \subseteq \conv \L(\smace{A})$,
  which completes the proof.
\end{proof}

%%%%%%%%%%%%%%%%%%%%%%%%%%%%%%%%%%%%%%%%%%%%%%%%%%%%%%%%%%%%%%%
% INNER APPROXIMATION ALGORITHMS
%%%%%%%%%%%%%%%%%%%%%%%%%%%%%%%%%%%%%%%%%%%%%%%%%%%%%%%%%%%%%%%
\section{Inner approximation algorithms}\label{sAlg}

Theorem~\ref{thmBdDes} naturally yields an algorithm to compute a very
sharp inner approximation of $\L(\smace{A})$, which could also be exact in some cases.
We will present the algorithm in the sequel (Section~\ref{ssSubMatEnum}). 
First, we define some notions and
propose two simple but useful methods for inner approximations.

Any subset of $\mna{S}$ is called an \emph{inner
  approximation}. Similarly, any set that contains $\mna{S}$ is called
an \emph{outer approximation}.  In our case, an inner approximation of the 
eigenvalue set $\ivr{\lambda}_i(\smace{A})$, is
denoted by
$\inum{\inner}_i(\smace{A})=[\uv{\inner}_i(\smace{A}),\ov{\inner}_i(\smace{A})]\subseteq\inum{\lambda}_i(\smace{A})$,
and an outer approximation is denoted by
$\inum{\outer}_i(\smace{A})=[\uv{\outer}_i(\smace{A}),\ov{\outer}_i(\smace{A})]\supseteq\inum{\lambda}_i(\smace{A})$,
where $1 \leq i \leq n$.

From a practical point of view, an outer approximation is usually more useful.
However, an inner approximation is also important in some applications.
For example, it could be used to measure quality
(sharpness) of an outer approximation, or it could be used to prove
(Hurwitz or Schur) unstability of certain interval matrices, cf. \cite{Roh1996}.

%%%%%%%%%%%%%%%%%%%%%%%%%%%%%%%%%%%%%%%%%%%%%%%%%%%%%%%%%%%%%%%
\subsection{Local improvement}

The first algorithm that we present is based on local improvement
search technique. A similar method, but for another class of symmetric
interval matrices was proposed by Rohn \cite{Roh1996}. The basic idea
of the algorithm is to start with an eigenvalue, $\lambda_i(\Mid{A})$,
and the corresponding eigenvector, $\vr{v}_i(\Mid{A})$, of the
midpoint matrix, $\Mid{A}$, and then move to an extremal matrix in
$\smace{A}$ according to the sign pattern of the eigenvector. The
procedure is repeated until no improvement is possible.

Algorithm~\ref{algInnerLocal} outputs the upper boundaries
$\ov{\inner}_i(\smace{A})$ of the inner approximation
$[\uv{\inner}_i(\smace{A}),\ov{\inner}_i(\smace{A})]$, where $1 \leq i
\leq n$. The lower boundaries, $\uv{\inner}_i(\smace{A})$, can be
obtained similarly. The validity of the procedure follows from the
fact that every considered matrix, $\mace{A}$, belongs to $\smace{A}$.

\begin{algorithm}[h]
\caption{(Local improvement for $\ov{\inner}_i(\smace{A})$, $i=1,\ldots,n$)
\label{algInnerLocal}}
\begin{algorithmic}[1]
\FOR{$i=1,\ldots,n$}
\STATE
$\ov{\inner}_i(\smace{A})=-\infty$;
\STATE 
$\mace{A}:=\Mid{A}$;
\WHILE{$\lambda_i(\mace{A})>\ov{\inner}_i(\smace{A})$}
\STATE
$\mace{D}:=\diag{\sgn{\vr{v}_i(\mace{A})}}$;
\STATE
$\mace{A}:=\Mid{A}+\mace{D}\Rad{A}\mace{D}$;
\STATE
$\ov{\inner}_i(\smace{A}):=\lambda_i(\mace{A})$;
\ENDWHILE
\ENDFOR
\RETURN
$\ov{\inner}_i(\smace{A}),\ i=1,\ldots,n$.
\end{algorithmic}
\end{algorithm}

The algorithm terminates after, at most, $2^n$ iterations.  However,
usually in practice the number of iterations is much smaller, which
makes the algorithm attractive for applications. Our numerical
experiments (Section~\ref{sNumer}) indicate 
% \footnote{@Milan: We should refer to our experiment section 
%   here. Very important claimQ.\\ Reply: What is claimQ?\\
% Q is a typo. I rephrase the question: Is it ok that I mention Sec 5?} 
that the number of iterations is rarely
greater than two, even for matrices of dimension 20. Moreover, the
resulting inner approximation is quite sharp, depending on the width
of intervals in $\smace{A}$. This is not surprising as whenever the
input intervals are narrow enough, the algorithm produces, sometimes
even after the first iteration, exact bounds; see \cite{Dei1991}.  We
refer the reader to Section~\ref{sNumer} for a more detailed
presentation of the experiments.

%%%%%%%%%%%%%%%%%%%%%%%%%%%%%%%%%%%%%%%%%%%%%%%%%%%%%%%%%%%%%%%
\subsection{Vertex enumeration}

The second method that we present is based on vertex enumeration. 
It consists of inspecting all matrices 
$$ 
\mace{A}_z:=\Mid{A}+\diag{\vr{z}}\Rad{A}\diag{\vr{z}},
\ \ \vr{z}\in\{\pm1\}^n,\ \ z_1=1,
$$
and continuously improving an inner approximation
$\ov{\inner}_i(\smace{A})$, whenever
$\lambda_i(\mace{A}_z)>\ov{\inner}_i(\smace{A})$, where $1 \leq i \leq
n$.  The lower bounds, $\uv{\inner}_i(\smace{A})$,  could
be obtained in a similar way using the matrices
$\Mid{A}-\diag{\vr{z}}\Rad{A}\diag{\vr{z}}$, where $\vr{z}\in\{\pm1\}^n$,
and $z_1=1$. The condition $z_1=1$ follows from the fact that
$\diag{\vr{z}}\Rad{A}\diag{\vr{z}}=\diag{-\vr{z}}\Rad{A}\diag{-\vr{z}}$,
which gives us the freedom to fix one component of $\vr{z}$.  The
number of steps that the algorithm performs is $2^{n-1}$.  Therefore,
this method is suitable only for matrices of moderate dimensions.

The main advantages of the {\em vertex enumeration} approach are the
following. First, it provides us with sharper inner approximation of the eigenvalue sets than the local improvement. Second,
two of the computed bounds are exact; by Hertz \cite{Her1992}
(cf. \cite{Roh2005}) and Hertz \cite{Her2009} we have that
$\ov{\inner}_1(\smace{A})=\ov{\lambda}_1(\smace{A})$ and
$\uv{\inner}_n(\smace{A})=\uv{\lambda}_n(\smace{A})$. 
The other bounds calculated by the vertex enumeration,
even though it was conjectured that there were exact  \cite{QiuWan2005},
it turned out that they were not exact, in general \cite{YuaHe2008}
% \footnote{@Milan: I can not understand this sentence. Please rephrase\\Reply: OKAY? ok}
The assertion by Hertz \cite[Theorem 1]{Her2009} that $\uv{\inner}_1(\smace{A})=\uv{\lambda}_1(\smace{A})$ and
$\ov{\inner}_n(\smace{A})=\ov{\lambda}_n(\smace{A})$ is wrong, too; see Example~\ref{exmSing}. Nevertheless, Theorem~\ref{thmBdDes} and its proof
indicate a sufficient condition: If no eigenvector corresponding to
an eigenvalue of $\smace{A}$ has a
zero component, then the vertex enumeration yields exactly the
eigenvalue sets $\inum{\lambda}_i(\smace{A})$, $i=1,\dots,n$.

The efficient implementation of this approach is quite challenging.
In order to overcome in practice the exponential complexity of the
algorithm, we implemented a branch \& bound algorithm, which is in
accordance with the suggestions of Rohn \cite{Roh1996}. However, the
adopted bounds are not that tight, and the actual running times are usually 
worse than the direct vertex enumeration. That is why we do not
consider further this variant. The direct vertex enumeration
scheme for computing the upper bounds,  $\ov{\inner}_i(\smace{A})$,
is presented in Algorithm~\ref{algInnerExp}.

\begin{algorithm}[ht]
\caption{(Vertex enumeration for $\ov{\inner}_i(\smace{A})$, $i=1,\ldots,n$)
\label{algInnerExp}}
\begin{algorithmic}[1]
\FOR{$i=1,\ldots,n$}
\STATE
$\ov{\inner}_i(\smace{A})=\lambda_i(\Mid{A})$;
\ENDFOR
\FORALL{$\vr{z}\in\{\pm1\}^n,\ z_1=1,$}
\STATE
$\mace{A}:=\Mid{A}+\diag{z}\Rad{A}\diag{z}$;
\FOR{$i=1,\ldots,n$}
\IF{$\lambda_i(\mace{A})>\ov{\inner}_i(\smace{A})$}
\STATE
$\ov{\inner}_i(\smace{A}):=\lambda_i(\mace{A})$;
\ENDIF
\ENDFOR
\ENDFOR
\RETURN
$\ov{\inner}_i(\smace{A}),\ i=1,\ldots,n$.
\end{algorithmic}
\end{algorithm}

%%%%%%%%%%%%%%%%%%%%%%%%%%%%%%%%%%%%%%%%%%%%%%%%%%%%%%%%%%%%%%%
\subsection{Submatrix vertex enumeration}\label{ssSubMatEnum}

In this section we present an algorithm that is based on
Theorem~\ref{thmBdDes}, and it usually produces very tight inner
approximations, even exact ones in some cases.
The basic idea the algorithm is to enumerate all the 
vertices of all the principal submatrices of $\smace{A}$. 
The number of steps performed with this approach is 
$$
2^{n-1}+n2^{n-2}+\binom{n}{2}2^{n-2}+\dots+n2^0=\frac{1}{2}\left(3^n-1\right).
$$
To overcome the obstacle of the exponential number of iterations, at
least in practice, we notice that not all eigenvalues of the principal
submatrices of the matrices in $\smace{A}$ belong to some of the
eigenvalue sets $\inum{\lambda}_i(\smace{A})$, where $1 \leq i \leq n$. 
For this we will introduce a condition for checking such an inclusion.

Assume that we are given an inner approximation
$\inum{\inner}_i(\smace{A})$ and an outer approximation
$\inum{\outer}_i(\smace{A})$ of the eigenvalue sets
$\inum{\lambda}_i(\smace{A})$; that is
$\inum{\inner}_i(\smace{A})\subseteq\inum{\lambda}_i(\smace{A})\subseteq\inum{\outer}_i(\smace{A})$,
where $1 \leq i \leq n$.  As we will see in the sequel, the quality of
the output of our methods depends naturally on the sharpness of the
outer approximation used.
%For the recent results in the field of outer approximation see e.g. \cite{HlaDan2008b}. 

Let $\smace{D}\subset\R^{k\times k}$ be a principal submatrix of
$\smace{A}$ and, without loss of generality, assume that it is situated
in the right-bottom corner, i.e.,
\begin{align*}
\smace{A}=
\begin{pmatrix}\smace{B}&\imace{C}\\
\imace{C}^T&\smace{D}\end{pmatrix},
\end{align*}
where $\smace{B}\subset\R^{(n-k)\times(n-k)}$ and
$\imace{C}\subset\R^{(n-k)\times k}$. Let $\lambda$ be an eigenvalue
of some vertex matrix $\mace{D}\in\smace{D}$ which is of the form
\nref{pfFormD}, and let $\vr{y}$ be the corresponding eigenvector. If
the eigenvector is not unique then $\lambda$ is a multiple eigenvalue
and therefore it is a simple eigenvalue of some principal
submatrix of $\smace{D}$; in this case we restrict our consideration
to this submatrix.
 
We want to determine whether $\lambda$ is equal to
$\ovr{\lambda}_p(\smace{A})\in\L(\smace{A})$ for some fixed
$p\in\seznam{n}$, or if this is not possible, to improve the upper bound
$\ov{\inner}_p(\smace{A})$; the lower bound can be handled
accordingly. In view of \nref{pfLamC}, it must be
$$\vr{0}\in\imace{C}\vr{y} \enspace,$$
so that $\lambda$ to be an eigenvalue  
%and $(\vr{0}^T,\vr{y}^T)^T$ the corresponding eigenvector 
of some matrix in $\smace{A}$. Now, we are sure that
$\lambda\in\L(\smace{A})$ and it remains to determine whether
$\lambda$ also belongs to $\inum{\lambda}_p(\smace{A})$. 

%It is difficult to give a uviversal response, we tackle the problem in
%the following way.

If $\lambda\leq \ov{\inner}_p(\smace{A})$, then it is useless to
further considering $\lambda$, as it will not extend the inner
approximation of the $p$th eigenvalue set. If $p=1$ or
$\lambda<\uv{\outer}_{p-1}(\smace{A})$, then $\lambda$ must belong to
$\ovr{\lambda}_p(\smace{A})$, and we can improve the inner bound
$\ov{\inner}_p(\smace{A}):=\lambda$. In this case the algorithm terminates early,
and that is the reason we need $\inum{\outer}_i(\smace{A})$ to be as tight as possible,
where $1 \leq i \leq n$.

If $p>1$ and $\lambda\geq\uv{\outer}_{p-1}(\smace{A})$, we proceed as
follows.  We pick an arbitrary $\mace{C}\in\imace{C}$, such that
$\mace{C}\vr{y}=\vr{0}$; we refer to, e.g.~\cite{Roh2006} for details
on the selection process. Next, we select an arbitrary
$\mace{B}\in\smace{B}$ and let
\begin{align}\label{dfAidentLam}
\mace{A}:=
\begin{pmatrix}\mace{B}&\mace{C}\\
\mace{C}^T&\mace{D}\end{pmatrix}.
\end{align}
We compute the  eigenvalues of $A$, and if $\ov{\inner}_p(\smace{A})<\lambda_p(\mace{A})$,  then we set $\ov{\inner}_p(\smace{A}):=\lambda_p(\mace{A})$, otherwise we do nothing. 

However, it can happen that $\lambda=\ov{\lambda}_i(\smace{A})$, and
we do not identify it, and hence we do not enlarge the inner
estimation $\inner_p(\smace{A})$. Nevertheless, if we apply the method
for all $p=1,\dots,n$ and all principal submatrices of $\smace{A}$, then  we
touch all the boundary points of $\L(\smace{A})$. 
If $\lambda\in\partial \L(\smace{A})$, then $\lambda$ is covered by the
resulting inner approximation. In the case when
$\lambda$ is an upper boundary point, we consider the maximal
$i\in\seznam{n}$ such that $\lambda=\ov{\lambda}_i(\smace{A})$ and
then the $i$th eigenvalue of the matrix \nref{dfAidentLam} must be
equal to $\lambda$. Similarly test are valid for a lower boundary point.

%It is simply true because if we take the maximal $i\in\seznam{n}$ such that $\lambda=\ov{\lambda}_i(\smace{A})$ (in the case when $\lambda$ is an upper boundary point)

%In the case when $\lambda$ is a boundary eigenvalue $\ovr{\lambda}_p(\smace{A})$ then $\lambda_p(\mace{A})=\ov{\lambda}_p(\smace{A})$ and we have achieved this exact bound (but we cannot conclude it yet).

Now we have all the ingredients at hand for the direct version of the
submatrix vertex enumeration approach that is presented in
Algorithm~\ref{algInnerSub}, which improves the upper bound
$\ov{\inner}_p(\smace{A})$ of an inner approximation, where the index
$p$ is still fixed. Let us also mention that in
step~\ref{algInnerSubDecomp} of Algorithm~\ref{algInnerSub}, the decomposition of
$\smace{A}$ according to the index set $J$ means that $\smace{D}$ is
a restriction of $\smace{A}$ to the rows and the columns indexed by $J$,
$\smace{B}$ is a restriction of $\smace{A}$ to the rows and the columns
indexed by $\seznam{n}\setminus J$, and $\imace{C}$ is a restriction
of $\smace{A}$ to the rows indexed by $\seznam{n}\setminus J$ and the
columns indexed by $J$.

\begin{algorithm}[ht]
\caption{(Direct submatrix vertex enumeration for $\ov{\inner}_p(\smace{A})$)
\label{algInnerSub}}
\begin{algorithmic}[1]
\STATE
compute outer approximation $\inum{\outer}_i(\smace{A})$, $i=1,\dots,n$;
\STATE
call Algorithm~\ref{algInnerLocal} to get inner approximation $\inum{\inner}_i(\smace{A})$, $i=1,\dots,n$;
%\FOR{$i=1,\ldots,n$}
\FORALL{$J\subseteq\seznam{n},\ J\not=\emptyset$,}\label{algInnerSubForI}
\STATE\label{algInnerSubDecomp}
decompose $\smace{A}=\left(\begin{smallmatrix}\smace{\simace{B}}&\simace{C}\\
\simace{C}^T&\smace{\simace{D}}\end{smallmatrix}\right)$ according to $J$;
\FORALL{$\vr{z}\in\{\pm1\}^{|J|},\ z_1=1,$}
\STATE
$\mace{D}:=\Mid{D}+\diag{z}\Rad{D}\diag{z}$;
\FOR{$i=1,\ldots,|J|$}
\STATE
$\lambda:=\lambda_i(\mace{D})$;
\STATE
$\vr{y}:=\vr{v}_i(\mace{D})$;
\IF{$\lambda>\ov{\inner}_p(\smace{A})$ and $\lambda\leq\ov{\outer}_p(\smace{A})$ and $\vr{0}\in\imace{C}\vr{y}$}
\IF{$p=1$ or $\lambda<\uv{\outer}_{p-1}(\smace{A})$}
\STATE
$\ov{\inner}_p(\smace{A}):=\lambda$;
\ELSE
\STATE
find $\mace{C}\in\imace{C}$ such that $\mace{C}\vr{y}=\vr{0}$;
\STATE
$\mace{A}:=
\left(\begin{smallmatrix}\Mid{B}&\mace{C}\\
\mace{C}^T&\mace{D}\end{smallmatrix}\right)$;
\IF{$\lambda_p(\mace{A})>\ov{\inner}_p(\smace{A})$}
\STATE
$\ov{\inner}_p(\smace{A}):=\lambda_p(\mace{A})$;
\ENDIF
\ENDIF
\ENDIF
\ENDFOR
\ENDFOR
\ENDFOR
%\ENDFOR
\RETURN
$\ov{\inner}_p(\smace{A})$.
\end{algorithmic}
\end{algorithm}

\subsubsection{Branch \& bound improvement}\label{sssSubMatBB}

In order to tackle the exponential worst case complexity of
Algorithm~\ref{algInnerSub}, we propose the following modification.
Instead of inspecting all non-empty subsets of $\seznam{n}$ in
step~\ref{algInnerSubForI}, we exploit a branch \& bound method, which
may skip some useless subsets. Let a non-empty $J\subseteq\seznam{n}$
be given. The new, possible improved, eigenvalue $\lambda$ must lie in
the interval
$\inum{\lambda}:=[\ov{\inner}_p(\smace{A}),\ov{\outer}_p(\smace{A})]$.
If this is the case, then the interval matrix
$\smace{A}-\inum{\lambda}\mace{I}$ must be irregular, i.e., it
contains a singular matrix. Moreover, the interval system
$$(\smace{A}-\inum{\lambda}\mace{I})\vr{x}=\vr{0},\ \ \|\vr{x}\|_\infty=1 
\enspace ,$$ has a solution $\vr{x}$, where $x_i=0$ for all $i\not\in
J$. 
We decompose $\smace{A}-\inum{\lambda}\mace{I}$ according to $J$, and, without loss of
generality, we may assume that $J=\sseznam{n-|J|+1}{n}$, then 

\begin{displaymath}
  \smace{A}-\inum{\lambda}\mace{I}=
  \begin{pmatrix}
    \smace{B}-\inum{\lambda}\mace{I}&\imace{C}\\
    \imace{C}^T&\smace{D}-\inum{\lambda}\mace{I}
  \end{pmatrix}.
\end{displaymath}

The interval system becomes
\begin{align}\label{eqIntRedCD}
\imace{C}\vr{y}=\vr{0},\ 
(\smace{D}-\inum{\lambda}\mace{I})\vr{y}=\vr{0},\  
\|\vr{y}\|_\infty=1,
\end{align}
where we considered $\vr{x}=(\vr{0}^T,\vr{y}^T)^T$. This is a
very useful necessary condition. If \nref{eqIntRedCD} has no solution,
then we cannot improve the current inner approximation. We can also
prune the whole branch with $J$ as a root; that is, we will inspect
no index sets $J'\subseteq J$. The strength of this condition follows
from the fact that the system \nref{eqIntRedCD} is overconstrained,
it has more equations than variables. Therefore, with high probability, 
that it has no solution, even for larger $J$.

%%%% Up to here
Let us make two comments about the interval system \nref{eqIntRedCD}. First,
this system has a lot of dependencies. They are caused from the 
multiple occurrences of $\inum{\lambda}$, and by the symmetry of
$\smace{D}$. If no solver for interval systems that can handle  dependencies is
available, then we can solve \nref{eqIntRedCD} as an ordinary interval
system, \quo{forgetting} the dependencies. The necessary condition
will be weaker, but still valid. This is what we did in our implementation.

The second comment addresses the expression $\|\vr{y}\|_\infty=1$. We
have chosen the maximum norm to pertain linearity of the interval
system. The expression could be rewritten as $-\vr{1}\leq \vr{y}\leq
\vr{1}$ (for checking solvability of \nref{eqIntRedCD} we can use
either normalization $\|\vr{y}\|_\infty=1$ or $\|\vr{y}\|_\infty\leq
1$). Another possibility is to write
$$-\vr{1}\leq \vr{y}\leq \vr{1},\ y_i=1\mbox{ for some } i\in\seznam{|J|}.$$
This indicates that we can split the problem into solving $|J|$ interval systems
\begin{align*}
\imace{C}\vr{y}=\vr{0},\ 
(\smace{D}-\inum{\lambda}\mace{I})\vr{y}=\vr{0},\  
-\vr{1}\leq \vr{y}\leq \vr{1},\ y_i=1 
\enspace ,
\end{align*}
where $i$ runs, sequentially, through all the values
$\{1,\dots,|J|\}$; cf. the ILS method proposed in
\cite{HlaDan2008}. The advantage of this approach is that the overconstrained interval
systems have (one)  more  equation than the original overconstrained system, and
hence the resulting necessary condition could be stronger. Our numerical
results discussed in Section~\ref{sNumer} concern this variant. As a
solver for interval systems we utilize the convex approximation approach
by Beaumont \cite{Bea1998}; it is sufficiently fast and produces
narrow enough approximations of the solution set.

\subsubsection{How to conclude for exact bounds?}\label{ssExact}

Let us summarize properties of the submatrix vertex enumeration method.
On the one hand the worst case complexity of the algorithm is rather
prohibitive, $O(3^{n})$, but on the other, we obtain better inner
approximations, and sometimes we get exact bounds of the eigenvalue
sets.  Theorem~\ref{thmBdDes} and the discussion in the previous
section allow us to recognize exact bounds.  For any
$i\in\sseznam{2}{n}$, we have that if
$\ov{\lambda}_i(\smace{A})<\uv{\lambda}_{i-1}(\smace{A})$, then
$\ov{\inner}_i(\smace{A})=\ov{\lambda}_i(\smace{A})$; a similar
inequality holds for the lower bound. This is a rather theoretical
recipe because we may not know a priori whether the assumption is
satisfied. However, we can propose a sufficient condition: If
$\ov{\outer}_i(\smace{A})<\uv{\outer}_{i-1}(\smace{A})$, then the
assumption is obviously true, and we conclude
$\ov{\inner}_i(\smace{A})=\ov{\lambda}_i(\smace{A})$; otherwise we
cannot conclude.

This sufficient condition is another reason why we need a sharp outer
approximation. The sharper it is, the more times we are able to
conclude that the exact bound is achieved.

Exploiting the condition we can also decrease running time of
submatrix vertex enumeration. We call
Algorithm~\ref{algInnerSub} only for $p\in\seznam{n}$ such that $p=1$
or $\ov{\outer}_p(\smace{A})<\uv{\outer}_{p-1}(\smace{A})$. The
resulting inner approximation may be a bit less tight, but the number
of exact boundary points of $\L(\smace{A})$ that we can identify
remains the same.

Notice that there is enough open space for developing better
conditions. For instance, we do not know whether
$\ov{\inner}_i(\smace{A})<\uv{\inner}_{i-1}(\smace{A})$ (computed by
submatrix vertex enumeration) can serve also as a sufficient condition
for the purpose of determining exact bounds.

%%%%%%%%%%%%%%%%%%%%%%%%%%%%%%%%%%%%%%%%%%%%%%%%%%%%%%%%%%%%%%%
% NUMERICAL EXPERIMENTS
%%%%%%%%%%%%%%%%%%%%%%%%%%%%%%%%%%%%%%%%%%%%%%%%%%%%%%%%%%%%%%%
\section{Numerical experiments}\label{sNumer}

In this section we present some examples and numerical results
illustrating properties of the proposed algorithms.  The experiments we
performed on a PC \texttt{Intel(R) Core 2}, CPU 3~GHz, 2~GB RAM, and
the source code was written in \texttt{C++}. We use
\texttt{GLPK~v.4.23} \cite{GLPK} for solving linear programs,
\texttt{CLAPACK~v.3.1.1} for its linear algebraic routines, and
\texttt{PROFIL/BIAS~v.2.0.4}~\cite{Profil} for interval arithmetic and
basic operations. Notice, however, that routines of \texttt{GLPK} and
\texttt{CLAPACK}\cite{Clapack} do not produce verified solutions; for
real-life problems this may not be acceptable.

\begin{example}\label{exSymUnsym} 
  Consider the following symmetric interval matrix
  \begin{align*} 
    \smace{A}=
    \begin{pmatrix} 1& 2&[1,5]\\ 2& 1&1\\ [1,5]   &1& 1\end{pmatrix}^S.
  \end{align*} 

  Local improvement (Algorithm~\ref{algInnerLocal}) yields
  an inner approximation
  \begin{align*} 
    \inum{\inner}_1(\smace{A})&=[3.7321,\,6.7843],\\
    \inum{\inner}_2(\smace{A})&=[0.0888,\,0.3230],\\
    \inum{\inner}_3(\smace{A})&=[-4.1072,\,-1.0000].
  \end{align*} 
  The same result is obtained by the vertex enumeration
  (Algorithm~\ref{algInnerExp}). Therefore,
  $\ov{\inner}_1(\smace{A})=\ov{\lambda}_1(\smace{A})$ and
  $\uv{\inner}_3(\smace{A})=\uv{\lambda}_3(\smace{A})$.  An outer
  approximation that is needed by the submatrix vertex enumeration
  (Algorithm~\ref{algInnerSub}) is computed using the methods of Hlad\'{i}k
  et al. \cite{HlaDan2008b,HlaDan2008c}. It is 
  \begin{align*} 
    \inum{\outer}_1(\smace{A})&=[3.5230,\,6.7843],\\
    \inum{\outer}_2(\smace{A})&=[0.0000,\,1.0519],\\
    \inum{\outer}_3(\smace{A})&=[-4.1214,\,-0.2019].
  \end{align*} 
  Now, the submatrix vertex enumeration algorithm yields the inner
  approximation
  \begin{align*} 
    \inum{\inner}'_1(\smace{A})&=[3.7321,\,6.7843],\\
    \inum{\inner}'_2(\smace{A})&=[0.0000,\,0.3230],\\
    \inum{\inner}'_3(\smace{A})&=[-4.1072,\,-1.0000].
  \end{align*} 
  Since the outer approximation intervals do not overlap,
  we can conclude that this approximation is exact, that is,
  $\inum{\lambda}_i(\smace{A})=\inum{\inner}'_i(\smace{A})$, $i=1,2,3$.
  
  This example shows two important aspects of the interval eigenvalue
  problem. First, it demonstrates that the vertex enumeration does not produce
  exact bounds in general. Second, the symmetric eigenvalue set can be a
  proper subset of the unsymmetric one, i.e.,
  $\L(\smace{A})\subsetneqq\L(\imace{A})$. 
  This could be easily seen in the matrix
  \begin{align*}
    \begin{pmatrix} 1& 2&1\\ 2& 1&1\\5&1& 1\end{pmatrix}.
  \end{align*} 
  It has three real eigenvalues $4.6458$, $-0.6458$ and
  $-1.0000$, but the second one does not belong to
  $\L(\smace{A})$. Indeed, using the method by Hlad\'{i}k et
  al. \cite{HlaDan2008} we obtain
  \begin{align*}
    \L(\imace{A})=[3.7321,\,6.7843]\cup[-0.6458,\,0.3230]\cup[-4.1072,\,-1.0000].
  \end{align*}
\end{example}

\begin{example}
  Consider the example given by Qiu et al. \cite{QiuChe1996} (see also \cite{HlaDan2008b, YuaHe2008}):
  \begin{align*}
    \smace{A}=
    \begin{pmatrix}
      [2975,3025]& [-2015,-1985]& 0&0\\
      [-2015,-1985]&[4965,5035]& [-3020,-2980]& 0\\
      0& [-3020,-2980] &[6955,7045]& [-4025,-3975]\\
      0 &0& [-4025,-3975]& [8945,9055]\end{pmatrix}^S.
  \end{align*}
The local improvement (Algorithm~\ref{algInnerLocal}) yields an inner approximation
  \begin{align*}
    \inum{\inner}_1(\smace{A})=[12560.8377,\,12720.2273],&\ \ 
    \inum{\inner}_2(\smace{A})=[7002.2828,\,7126.8283],\\ 
    \inum{\inner}_3(\smace{A})=[3337.0785,\,3443.3127],\ \ \ &\ \ 
    \inum{\inner}_4(\smace{A})=[842.9251,\,967.1082].
  \end{align*}
The vertex enumeration (Algorithm~\ref{algInnerExp}) produces the same
  result. Hence we can state that $\ov{\inner}_1(\smace{A})$ and
  $\uv{\inner}_4(\smace{A})$ are optimal.

To call the last method, submatrix vertex enumeration
  (Algorithm~\ref{algInnerSub}) we need an outer approximation. We use
  the following by \cite{HlaDan2008b}
  \begin{align*}
    \inum{\outer}_1(\smace{A})=[12560.6296,\,12720.2273],&\ \ 
    \inum{\outer}_2(\smace{A})=[6990.7616,\,7138.1800],\\ 
    \inum{\outer}_3(\smace{A})=[3320.2863,\,3459.4322],&\ \ 
    \inum{\outer}_4(\smace{A})=[837.0637,\,973.1993].
  \end{align*}
  Now, submatrix vertex enumeration yields the same inner
  approximation as the previous methods. However, now we have more information.
  Since the outer approximation interval are mutually disjoint,
  the obtained results are the best possible. Therefore,
  $\inum{\inner}_i(\smace{A})=\inum{\lambda}_i(\smace{A})$,
  where $i=1,\dots,4$.
\end{example}

\begin{example}\label{exmSing}
Herein, we present two examples for approximating the singular values of an
  interval matrix. Let $\mace{A}\in\R^{m\times n}$ and $q:=\min\{m,n\}$. 
  By the Jordan--Wielandt theorem \cite{GolLoa1996,
    HorJoh1985, Mey2000} the singular values,
  $\sigma_1(\mace{A})\geq\dots\geq\sigma_q(\mace{A})$, of $\mace{A}$
  are identical with the $q$ largest eigenvalues of the symmetric
  matrix
  \begin{displaymath}
    \begin{pmatrix}
      \mace{0}&\mace{A}^T\\
      \mace{A}&\mace{0}
    \end{pmatrix}.
  \end{displaymath}
  Thus, if we consider the singular value sets
  $\inum{\sigma}_1(\imace{A}),\dots,\inum{\sigma}_q(\imace{A})$ of
  some interval matrix $\imace{A}\in\R^{m\times n}$, we can identify
  them as the $q$ largest eigenvalue sets of the symmetric interval
  matrix
  
  \begin{displaymath}
    \imace{M}:=
    \begin{pmatrix}\mace{0}&\imace{A}^T\\
      \imace{A}&\mace{0}
    \end{pmatrix}^S.
  \end{displaymath}

  (1) Consider the following interval matrix from \cite{Dei1991b}.
  \begin{align*}
    \imace{A}=
    \begin{pmatrix}
      [2,3]&[1,1]\\ [0,2]&[0,1]\\ [0,1]&[2,3]
    \end{pmatrix}
  \end{align*}
Both the local improvement and the vertex enumeration result the same inner approximation, i.e.
  $$
  \inum{\inner}_1(\imace{M})=[2.5616,\,4.5431],\ \ 
  \inum{\inner}_2(\imace{M})= [1.2120,2.8541].
  $$
  Thus, $\ov{\sigma}_1(\imace{A})=4.5431$. 
  Additionally, consider the following outer approximation from \cite{HlaDan2008b}.
  \begin{align*}
    \inum{\outer}_1(\imace{M})=[2.0489,\,4.5431],\ \ 
    \inum{\outer}_2(\imace{M})=[0.4239,\,3.1817].
  \end{align*}
  Using Algorithm~\ref{algInnerSub}, we obtain
  $$
  \inum{\inner}'_1(\imace{M})=[2.5616,\,4.5431],\ \ 
  \inum{\inner}'_2(\imace{M})= [1.0000,2.8541].
  $$
  Now we can claim that $\uv{\sigma}_2(\imace{A})=1$, since
  $\uv{\outer}_2(\imace{M})>0$. 
  Unfortunately, we cannot conclude about the 
  exact values of the remaining quantities, since the two outer
  approximation intervals overlap. We know only that
  $\uv{\sigma}_1(\imace{A})\in[2.0489,\,2.5616]$ and
  $\ov{\sigma}_2(\imace{A})\in[2.8541,\,3.1817]$.

  (2) The second example comes from Ahn \& Chen \cite{AhnChen2007}.
  Let $\imace{A}$ be the following interval matrix
  \begin{align*}
    \imace{A}=
    \begin{pmatrix}
      [0.75,\;2.25]& [-0.015,\;-0.005]& [1.7,\;5.1]\\
      [3.55,\;10.65]& [-5.1,\;-1.7]& [-1.95,\;-0.65]\\
      [1.05,\;3.15]& [0.005,\;0.015]& [-10.5,\;-3.5]
    \end{pmatrix}.
  \end{align*}
  Both local improvement and vertex enumeration yield the same result, i.e.
  \begin{align*}
    \inum{\inner}_1(\imace{M})&=[4.6611,\,13.9371],\ \ 
    \inum{\inner}_2(\imace{M}) =[2.2140,\,11.5077],\\ 
    \inum{\inner}_3(\imace{M})&=[0.1296,\,2.9117].
  \end{align*}
  Hence, $\ov{\sigma}_1(\imace{A})=13.9371$.
  As an outer approximation we use the following intervals, using \cite{HlaDan2008b}.
  \begin{align*}
    \inum{\outer}_1(\imace{M})&=[4.3308,\,14.0115],\ \ 
    \inum{\outer}_2(\imace{M})=[1.9305,\,11.6111],\\ 
    \inum{\outer}_3(\imace{M})&=[0.0000,\,5.1000].
  \end{align*}
Running the submatrix vertex enumeration, we get the inner approximation
  \begin{align*}
    \inum{\inner}'_1(\imace{M})&=[4.5548,\,13.9371],\ \ 
    \inum{\inner}'_2(\imace{M}) =[2.2140,\,11.5077],\\ 
    \inum{\inner}'_3(\imace{M})&=[0.1296,\,2.9517].
  \end{align*}
  % The second exact boundary point that we achieved is $\uv{\sigma}_3(\imace{A})=0.1296$. The others remain uncertain, but within the computed inner and outer approximation.
  We cannot conclude that
  $\uv{\sigma}_3(\imace{A})=\uv{\inner}_3(\imace{A})=0.1296$, because
  $\inum{\outer}_3(\imace{M})$ has a nonempty intersection with the
  fourth largest eigenvalue set, which is equal to zero. Also the
  other singular value sets remain uncertain, but within the computed
  inner and outer approximation.

Notice that $\uv{\inner}'_1(\imace{M})<\uv{\inner}_1(\imace{M})$, whence $\uv{\inner}'_1(\imace{M})<\uv{\lambda}_1(\imace{M})=\uv{\sigma}_1(\imace{A})$ disproving the Hertz's Theorem~1 from \cite{Her2009} that the lower and upper limits of $\inum{\lambda}_1(\imace{M})$ and $\inum{\lambda}_n(\imace{M})$ are computable by the vertex enumeration method. It is true only for $\ov{\lambda}_1(\imace{M})$ and $\uv{\lambda}_n(\imace{M})$. 
%for example [2.1496 -0.0053 2.7116; 3.55 -1.7 -0.65; 1.05  0.015 -3.5]
\end{example}

\begin{example}
  In this example we present some randomly generated examples of large
  dimensions. The entries of the midpoint matrix, $\Mid{\mace{A}}$,
  are taken randomly in $[-20,20]$ using the uniform distribution.
  The entries of the radius matrix $\Rad{\mace{A}}$ are taken
  randomly, using the uniform distribution in $[0,R]$, where $R$ is a
  positive real number. We applied our algorithm on the interval
  matrix $\imace{M}:=\imace{A}^T\imace{A}$, because it has a
  convenient distribution of eigenvalue set---some are overlapping and
  some are not. Sharpness of results is measured using the quantity
  \begin{displaymath}
    1-\frac{\vr{e}^T\Rad{\inner}(\smace{M})}{\vr{e}^T\Rad{\outer}(\smace{M})},
  \end{displaymath}
  where $\vr{e}=(1,\dots,1)^T$.  This quantity lies always within the
  interval $[0,1]$. The closer to zero it is, the tighter the
  approximation.  In addition, if it is zero, then we achieved exact
  bounds. The initial outer approximation,
  $\inum{\outer}_i(\smace{M})$, where $1 \leq i \leq n$, was computed
  using the method due of Hlad\'{i}k et al. \cite{HlaDan2008b}, and
  filtered by the method proposed by Hlad\'{i}k et al. in
  \cite{HlaDan2008c}.  Finally, it was refined according to the
  comment in Section~\ref{ssExact}. For the submatrix vertex
  enumeration algorithm we implemented the branch \& bound
  improvement, which is described in Section~\ref{sssSubMatBB}
  and~\ref{ssExact}.

  The results are displayed in Table~\ref{tabRand}. 
  % Each row stands for one example computed by all three presented algorithms.

  \begin{table}[phtb]%\footnotesize
    \begin{center}
      \begin{tabular}{|r|l|r|r|r|r|r|r|}
        \hline
        % \mbox{}\\
        \rule[-3pt]{0pt}{15pt}
        {$n$}\,{} & \ $R$ 
        &  \multicolumn{2}{|c|}{Algorithm~\ref{algInnerLocal}}
        &  \multicolumn{2}{|c|}{Algorithm~\ref{algInnerExp}}
        &  \multicolumn{2}{|c|}{Algorithm~\ref{algInnerSub}} \\ 
        \cline{3-8} 
        \rule[-4pt]{0pt}{15pt}
        &  & sharpness & time & sharpness & time & sharpness & time \\ \hline 
        5 & 0.001 &  0.05817 &  0.00 s & 0.05041 &  0.00 s & 0.00000  & 0.04 s \\
        5 & 0.01 &  0.07020 &  0.00 s & 0.05163 &  0.00 s & 0.00000  & 0.03 s \\
        5 & 0.1 & 0.26273 &  0.00 s &  0.23389 &  0.00 s &  0.17332 & 0.04 s \\
        5 & 1 & 0.25112 &  0.00 s &  0.23644 &  0.00 s & 0.20884 & 0.01 s \\
        10 & 0.001 & 0.08077 &  0.00 s & 0.07412 &  0.09 s & 0.00000  & 1.15 s \\
        10 & 0.01 &  0.13011 &  0.01 s & 0.11982 &  0.08 s &  0.04269 & 1.29 s \\
        10 & 0.1 &  0.27378 &  0.01 s & 0.25213 &  0.09 s &  0.12756 & 3.17 s \\
        10 & 1 & 0.56360  &  0.01 s & 0.52330 &  0.09 s &  0.52256 & 2.58 s \\
        15 & 0.001 &  0.07991 &  0.02 s & 0.07557 &  7.3 s &  0.00000 & 16.47 s \\
        15 & 0.01 &  0.21317 &  0.02 s & 0.19625 &  6.5 s &  0.11341 & 2 m 29 s \\
        15 & 0.1 & 0.36410 &  0.02 s & 0.34898 &  7.0 s & 0.34869  & 4 m 58 s \\
        15 & 1 &  0.76036 &  0.02 s & 0.73182 &  7.2 s & 0.73182  & 7.5 s \\
        20 & 0.001 & 0.09399 &  0.06 s & 0.09080 & 7 m 21 s & 0.00000 & 13 m 46 s \\
        20 & 0.01 & 0.24293 & 0.06 s & 0.22976 & 7 m  6 s & 0.12574 & 1 h 14 m 55 s \\
        20 & 0.1 & 0.24293 &  0.06 s & 0.22976 & 7 m 14 s & 0.12574 & 1 h 15 m 41 s \\
        20 & 1 & 0.82044 &  0.06 s & 0.79967 & 7 m 33 s & 0.79967 & 7 m 39 s \\
        25 & 0.001 & 0.14173 &  0.13 s & 0.13397 &  6 h 53 m 0 s &
        0.02871 & 9 h 32 m 54 s \\
        % 25 & 0.01 &  &  0.01 s &  &  0.09 s &   & 3 s \\
        % 30 & 0.001 &  &  0.01 s &  &  0.09 s &   & 3 s \\
        % 30 & 0.01 &  &  0.01 s &  &  0.09 s &   & 3 s \\
        \hline
      \end{tabular}
      \caption{Eigenvalues of random interval symmetric matrices
        $\imace{A}^T\imace{A}$ of dimension $n \times
        n$.\label{tabRand}}
    \end{center}
  \end{table}
\end{example}

\begin{example}
  In this example we present some numerical results on approximating
  singular value sets as introduced in Example~\ref{exmSing}. The
  input consists of an interval (rectangular) matrix
  $\imace{A}\subseteq\R^{m\times n}$ which is selected randomly as in
  the previous example.

  Table~\ref{tabRandSing} presents our experiments.  The time in the
  table corresponds to the computation of the approximation of only the
  $q$ largest eigenvalue sets of the Jordan--Wielandt matrix.  
  % so we save some running time.

  \begin{table}[phtb]%\footnotesize
    \begin{center}
      \begin{tabular}{|r|r|l|r|r|r|r|r|r|}
        \hline
        % \mbox{}\\
        \rule[-3pt]{0pt}{15pt}
        {$m$}\,{} & {$n$}\,{} & \ $R$ 
        &  \multicolumn{2}{|c|}{Algorithm~\ref{algInnerLocal}}
        &  \multicolumn{2}{|c|}{Algorithm~\ref{algInnerExp}}
        &  \multicolumn{2}{|c|}{Algorithm~\ref{algInnerSub}} \\ 
        \cline{4-9} 
        \rule[-4pt]{0pt}{15pt}
        && & sharpness & time & sharpness & time & sharpness & time \\ \hline 
        5& 5 & 0.01 & 0.08945 & 0.00 s & 0.07716 & 0.10 s & 0.00000 & 0.53 s \\
        5& 5 & 0.1  & 0.09876 & 0.01 s & 0.09270 & 0.08 s & 0.00000 & 0.73 s \\
        5& 5 & 1    & 0.43560 & 0.01 s & 0.31419 & 0.10 s & 0.26795 & 4.34 s \\
        5&10 & 0.01 & 0.11320 & 0.02 s & 0.10337 & 5.79 s & 0.00000 & 7.91 s \\
        5&10 & 0.1  & 0.13032 & 0.02 s & 0.12321 & 5.98 s & 0.00000 & 8.40 s \\
        5&10 & 1    & 0.35359 & 0.02 s & 0.33176 & 5.52 s & 0.22848 & 21.53 s \\
        5&15 & 0.01 & 0.10603 & 0.05 s & 0.09424 & 5 m 31 s & 0.00000 & 5 m 36 s \\
        5&15 & 0.1  & 0.17303 & 0.04 s & 0.16758 & 5 m 33 s & 0.00000 & 7 m 58 s \\
        5&15 & 1    & 0.46064 & 0.05 s & 0.39708 & 5 m 32 s & 0.31847 & 15 m 47 s \\
        10&10 & 0.01 & 0.10211 & 0.06 s & 0.09652 & 8 m 3 s & 0.00000 & 8 m 19 s \\
        10&10 & 0.1  & 0.13712 & 0.07 s & 0.13387 & 8 m 10 s & 0.00000 & 14 m 12 s \\
        10&10 & 1    & 0.39807 & 0.07 s & 0.35580 & 7 m 52 s &
        0.30279 & 26 h 48 m 38 s \\
        10&15 & 0.01 & 0.09561 & 0.12 s & 0.09116 & 5 h 51 m 53 s &
        0.00000 & 5 h 54 m 56 s \\
        % 10&15 & 0.1 &  &  0.01 s &  &  0.09 s &   & 3 s \\
        % 10&15 & 1 &  &  0.01 s &  &  0.09 s &   & 3 s \\
        \hline
      \end{tabular}
      \caption{Singular values of random interval matrices of
        dimension $m \times n$.\label{tabRandSing}}
    \end{center}
  \end{table}
\end{example}

%%%%%%%%%%%%%%%%%%%%%%%%%%%%%%%%%%%%%%%%%%%%%%%%%%%%%%%%%%%%%%%
% CONCLUSION
%%%%%%%%%%%%%%%%%%%%%%%%%%%%%%%%%%%%%%%%%%%%%%%%%%%%%%%%%%%%%%%
\section{Conclusion and future directions}
\label{sec:conclusion}

We proposed a new solution theorem for the symmetric interval
eigenvalue problem, which describes some of the boundary points of the
eigenvalue set. Unfortunately the complete characterisation is still a
challenging open problem.

We developed an inner approximation algorithm (submatrix vertex
enumeration), which in the case where the eigenvalue sets are
disjoint, and the intermediate gaps are wide enough, output exact
results.  To our knowledge, even under this assumption, this is the
first algorithm that can guarantee exact bounds.

Based on our numerical experiments suggest that the local search algorithm is superior to the 
submatrix vertex enumeration algorithm when the input matrices are not of very small dimension.

% confirmed that the submatrix vertex
% enumeration runs a long time for matrices of larger dimension and it
% can be applicable only for small matrices. However, local search
% algorithm runs very quickly and generally the inner approximation is
% not so underestimated.

\paragraph*{ Acknowledgment.}
ET is partially supported by an individual postdoctoral grant from the
Danish Agency for Science, Technology and Innovation.

%%%%%%%%%%%%%%%%%%%%%%%%%%%%%%%%%%%%%%%%%%%%%%%%%%%%%%%%%%%%%%%
% REFERENCES
%%%%%%%%%%%%%%%%%%%%%%%%%%%%%%%%%%%%%%%%%%%%%%%%%%%%%%%%%%%%%%%

\bibliographystyle{abbrv}
\bibliography{iev}

\tableofcontents

\end{document}